\documentclass{article}

\usepackage{spconf}
\usepackage{amsmath,graphicx}
\usepackage{amssymb,amsfonts,amsthm}
\usepackage{hyperref}
\usepackage[capitalize]{cleveref}
\usepackage{algorithm}
\usepackage{algorithmic}
\usepackage{booktabs}
\usepackage{multirow}
\usepackage{multicol}
\usepackage{placeins}

\usepackage{tikz}
\usepackage{pgfplots}
\usepackage{pgfgantt}

\title{Communication Efficient Private Federated Learning Using Dithering}

\name{Burak Hasırcıoğlu and Deniz Gündüz
}
\address{Information Processing and Communications Lab, Imperial College London, UK}
\newcommand{\f}{f_{\mathbf{w}}}
\newcommand{\LB}{\left(}
\newcommand{\RB}{\right)}

\newcommand{\U}{\mathrm{Unif}}
\newcommand{\norm}[2]{\left|\left| #2\right |\right|_{#1}}
\newcommand{\fg}{\nabla\ell(\mathbf{w}_t, d)}
\newcommand{\fgAv}[2]{\nabla \ell _{#1}^{#2}}
\newcommand{\fgAvEst}[2]{\hat{\nabla} \ell _{#1}^{#2}}
\newcommand{\g}{\mathbf{g}}
\newcommand{\round}[1]{\left \lceil #1 \right \rfloor}

\newtheorem{defn}{Definition}
\newtheorem{theorem}{Theorem}
\newtheorem{lemma}{Lemma}

\begin{document}
\onecolumn

%
\maketitle
\begin{abstract}
The task of preserving privacy while ensuring efficient communication is a fundamental challenge in federated learning. In this work, we tackle this challenge in the trusted aggregator model, and propose a solution that achieves both objectives simultaneously. We show that employing a quantization scheme based on subtractive dithering at the clients can effectively replicate the normal noise addition process at the aggregator. This implies that we can guarantee the same level of differential privacy against other clients while substantially reducing the amount of communication required, as opposed to transmitting full precision gradients and using central noise addition. We also experimentally demonstrate that the accuracy of our proposed approach matches that of the full precision gradient method.
\end{abstract}

\section{Introduction}
Federated learning (FL) framework allows multiple clients to collaboratively learn a model with the help of a parameter server (PS), without sharing their local datasets \cite{mcmahan17a}. Instead, clients share their local updates with the PS after local training round and the PS broadcasts the aggregated model back to the clients. Given the size of modern neural network architectures, this brings a massive amount of communication overhead. Hence, one core challenge in FL is the communication cost.

Privacy is a significant concern when it comes to machine learning (ML) since the solutions depend on data that can reveal sensitive information about its owner. Hence, while training ML models, it is vital to prevent any sensitive features from the training set from leaking. Privacy is one of the core promises of FL since data never leaves clients and only the local model updates are shared. Unfortunately, it has been shown that such updates, and even the final trained model are enough to reveal sensitive information about the training set \cite{geiping2020neurips, carlini2019secret, haim2022reconstructing, balle2018improving}. Hence, privacy protection is still a concern in FL and has been the topic of ongoing research. 

In this paper, we use differential privacy (DP) \cite{dwork2014algorithmic} to measure the privacy leakage, which has become the gold standard since its introduction. When an algorithm receives two adjacent datasets as inputs, DP measures the degree of similarity between the resulting outputs, and hence, it is a measure of indistinguishability. We give its formal definition as follows.

\begin{defn}
A randomized algorithm $\mathcal{M}$ satisfies
$(\varepsilon,\delta)$-DP for $\varepsilon>0$ and $\delta\in[0,1)$, if and only if 
\begin{equation}
    \Pr[\mathcal{M}(\mathcal{D})\in \mathcal{O}]\leq e^{\varepsilon}\Pr[\mathcal{M}(\mathcal{D'})\in \mathcal{O}]+\delta,
\end{equation}
for all datasets $\mathcal{D}$ and $\mathcal{D'}$ differing in only one sample, and for all possible subsets $\mathcal{O}$ of $\mathcal{M}$'s range.
\end{defn}

The value of $\varepsilon$ in the definition indicates the extent of privacy loss. A smaller $\varepsilon$ signifies a more robust privacy guarantee. Additionally, $\delta$ quantifies the probability of failure in the guarantee. Therefore, to ensure a strong privacy guarantee, $\delta$ must approach zero.

In this work, we jointly address privacy and communication efficiency in the trusted aggregator model. Hence, we aim to provide central DP guarantees; that is, we want the average of the client updates to satisfy the DP guarantees. The trusted aggregator model covers the scenarios in which the PS is trusted by the clients. This can be the case when the clients give their data to a trusted organization, such as a government agency, an independent regulator, or a research institution while they do not want to reveal their data to other clients via model updates, or to third parties via the final deployed model. However, a trusted aggregator model may not always require a trusted PS. For example, trusted execution environments (TEE) can be employed at the PS and data encryption between TEE and the clients and the code run at the TEE can be formally verified \cite{mo2021ppfl, nguyen2022federated}. Another application area of our scheme is distributed learning, such as the settings in \cite{horvath2022natural}. In such settings, a massive amount of data belonging to the same entity is used to train a model. So, the training task is distributed across many GPUs or servers, called worker terminals owned by the same entity. In order to avoid privacy leakage from the final deployed model, which can be accessed by third parties either via white-box or black-box access, the training procedure must satisfy DP. In such cases, our scheme significantly reduces the communication cost from GPUs or the worker terminals to the PS. Hence, the trusted aggregator model in this paper is only an abstraction but not a stringent system requirement, and applies to many scenarios encountered in practice.

Communication-efficient FL has been an active research area \cite{alistarh2017qsgd, suresh2017distributed, horvath2022natural}. However, when it comes to DP guarantees, directly extending these techniques is suboptimal since compression for communication efficiency and privacy introduce separate errors. Works such as \cite{amiri2021compressive, triastcyn2021dp, chaudhuri2022privacy, shah2022optimal, asi2023fast, isik2023exact, lang2023joint} consider tackling these two problems jointly. However, they mostly aim at guaranteeing local DP, in which each update from clients is separately protected. Unfortunately, such a stringent requirement considerably hurts the final model's performance. Moreover, due to the use of specially designed mechanisms to satisfy local DP, the aforementioned methods are not directly extendable to central DP. 

The main idea of our proposed method is that the randomization required for privacy hurts the accuracy, and hence, it may not be necessary for clients to send full precision updates since they will be already destroyed by the PS to some extent after adding noise. Instead, we propose using subtractive dithering quantization on the client updates prior to sending them to the PS. This reduces the communication cost while keeping the same accuracy. We show that if the quantization step size is randomly generated following a particular distribution, with the help of shared randomness between each client and the PS, the noise in the reconstructed update at the PS follows a normal distribution for any third party that does not have access to the common randomness. By employing dithered quantization, we simulate the normal noise addition process to ensure DP and avoid adding noise twice for quantization and DP, while significantly reducing the communication overhead.


\section{Problem Setting}
We consider FL with $N$ clients and a PS. Each client $i$ has its own dataset, $\mathcal{D}_i$ and they collaboratively learn a model $\f \in \mathbb{R}^m$ by minimizing a cumulative loss function, i.e.,

\begin{equation}
    \mathbf{w} = \arg\min \frac{1}{N} \sum_{i\in [N]} \frac{1}{|\mathcal{D}_i|} \sum_{d\in \mathcal{D}_i} \ell\LB \f \LB d_f\RB, d_l\RB, 
\end{equation}
where $\ell$ is the local loss function, $d_f$ and $d_l$ are the features and the label of the data point $d$, i.e, $d=(d_f,d_l)$. In the sequel, for brevity, we write $\ell(\mathbf{w},d) \triangleq \ell\LB \f \LB d_f\RB, d_l\RB$.

For training, we consider distributed stochastic gradient descent (SGD) optimization. That is, at each iteration $t$, each available client $i$ samples a small batch, $\mathcal{B}_{i,t} \in \mathcal{D}_i$, of average size $B$, and for each sample in the batch, computes the gradients of the loss function with respect to the current model parameters, i.e., $\fg$, $d\in \mathcal{B}_{i,t}$. Then, each client sends the average of the sample gradients, i.e., $\fgAv{i}{t} \triangleq \frac{1}{B} \sum_{d \in \mathcal{B}_{i,t}} \fg$, to the PS, where these gradients are further averaged to obtain the global average, denoted by $\g$. Finally, the new global model, which is updated by $\g$, is broadcast to all available clients for the consecutive update round.

\textbf{Threat Model:} In this work, we stick to the trusted aggregator model, i.e., the PS is trusted. Moreover, for compression, we assume that there are separate sources of common randomness shared between each client and the PS. Clients are assumed to be honest but curious. That is, they adhere to the protocol but may try to infer sensitive client information from average updates received from the PS. Hence, one of our goals is to protect the privacy of each client's local dataset from other clients since the updated model across rounds may reveal important sensitive information. Besides, once the training is completed, the final deployed model may leak sensitive information as well. Hence, we also aim to protect privacy leakage from the final deployed model, which makes our model and solution relevant even when the clients are trusted, as in distributed learning.

\section{Proposed Method}
In our proposed solution, we reduce the communication cost of each client update to the PS by quantizing them using subtractive dithering. We choose the step size of the quantization and the dithering parameters based on a gamma random variable. Such a trick achieves a quantization error that follows a Gaussian distribution, and hence, results in $(\varepsilon,\delta)$-DP guarantees, as shown in \cref{lem:gaussian_mechanism}.  Our solution uses the following fact about the scale mixture of uniform distributions, which appears in \cite{walker1999uniform}.

\begin{lemma}
\label{lem:uniform_mixture}
    If $\LB X|V=v  \RB\sim \U(\mu-\sigma\sqrt{v}, \mu + \sigma\sqrt{v})$, and $V \sim \Gamma[3/2,1/2]$, then $X \sim \mathcal{N}(\mu,\sigma^2)$, where $\Gamma[3/2,1/2]$ is the gamma distribution with shape and rate parameters 3/2 and 1/2, respectively.
\end{lemma}

This lemma states that if the realization of $V$, which follows a gamma distribution $\Gamma[3/2,1/2]$, is not known, then the distribution of $X$ becomes a normal distribution.

Our solution also utilizes the following fact about subtractive dithering \cite{lipshitz1992quantization, roberts1962picture}.

\begin{lemma}
\label{lem:dithering}
    Let $Y$ be the scalar to be quantized and $\hat{Y}= Q\LB Y + U \RB - U$, where $U \sim \U\LB -\frac{\Delta}{2}, \frac{\Delta}{2} \RB$ and $Q$ is the quantization function with step size $\Delta$. Then, $\hat{Y} = Y + U'$, where $U' \sim \U\LB -\frac{\Delta}{2}, \frac{\Delta}{2} \RB$ and independent from $U$.
\end{lemma}

We summarize our proposed method in \cref{alg:algorithm}. To generate batches $\mathcal{B}_{i,t}$, each client $i$ employs \textit{Poisson sampling}, that is, each sample in the local dataset $\mathcal{D}_i$ is sampled independently with probability $p$. Hence, $B=p|\mathcal{D}_i|$. To achieve formal privacy guarantees, at each client, we clip each sample gradient in the batch so that its $L_2$-norm is bounded by a parameter $C$. Then, the client computes the average of the sample gradients, $\fgAv{i}{t}$. Since $\norm{2}{\fg} \leq C$, each element of the vector $\fgAv{i}{t}$ also lies within the range $[-C,C]$. 

The client quantizes $\fgAv{i}{t}$ prior to sending it to the PS. Before quantization, for each element $(\fgAv{i}{t})_j$, $j\in [m]$, the client samples $V_{i,j}=v_{i,j}$ from the gamma distribution $\Gamma[3/2,1/2]$, and set the quantization step size $\Delta_{i,j}=2\sigma \sqrt{v_{i.j}}$. Hence, the client uses a separate step size parameter for every element of $\fgAv{i}{t}$. Accordingly, the representative quantization points are set as $\{\cdots, -\frac{3\Delta_{i,j}}{2}, -\frac{\Delta_{i,j}}{2}, \frac{\Delta_{i,j}}{2}, \frac{3\Delta_{i,j}}{2}, \cdots\}$. Then, for each element of $(\fgAv{i}{t})_j$, $j\in [m]$, it samples $U_{i,j}$ from $\U\LB-\frac{\Delta_{i,j}}{2}, \frac{\Delta_{i,j}}{2}\RB$ and quantizes $(\fgAv{i}{t})_j + U_{i,j}$. To be precise, we use the quantization function $Q(x) \triangleq \round{\frac{x-\Delta/2}{\Delta}}\Delta + \frac{\Delta}{2}$, where $\round{\cdot}$ is the function rounding its argument to the nearest integer.

To transmit the quantized gradients to the PS, client $i$ encodes each element $(\fgAv{i}{t})_j$ using $b_{i,j} \triangleq \left\lceil \log_2\LB 2\cdot \round{\frac{C}{\Delta_{i,j}}+1}\RB \right\rceil$ bits since $|(\fgAv{i}{t})_j|\leq C$, resulting in $\sum_{j\in[m]} b_{i,j}$
bits of communication from client $i$ to the PS in round $t$. We denote the message sent  by client $i$ in round $t$ by $\mathbf{m}_{i,j}$.

\begin{algorithm}[t]
\begin{small}
\begin{center}
\begin{algorithmic}
\textbf{Protocol in client $i$:}
\FOR{$t \in [T]$}
\STATE Receive $\mathbf{w}_{t-1}$ from the PS
\STATE Sample $\mathcal{B}_{i,t}$ from $\mathcal{D}_i$
\FOR{$d=(d_f, d_l) \in \mathcal{B}_{i,t}$} 
\STATE Calculate $\fg$
\STATE Clip: $\fg = \fg/\max\left\{1,\frac{\norm{2}{\fg}}{C}\right\}$
\ENDFOR
\STATE Calculate average gradient $\fgAv{i}{t}=\frac{1}{B} \sum_{d\in \mathcal{B}_{i,t}} \fg$
\FOR{$j \in [m]$}
\STATE Sample $v_{i,j} \sim \Gamma[3/2,1/2]$
\STATE Calculate $\Delta_{i,j} = 2\sigma \sqrt{v_{i,j}}$
\STATE Sample $U_{i,j} \sim \U\LB -\frac{\Delta_{i,j}}{2}, \frac{\Delta_{i,j}}{2} \RB$
\STATE Quantize: $\mathbf{m}_{i,j}=Q\LB(\fgAv{i}{t})_j+U_{i,j}\RB$ mapping it to the closest value in $\{\cdots, -\frac{3\Delta_{i,j}}{2}, -\frac{\Delta_{i,j}}{2}, \frac{\Delta_{i,j}}{2}, \frac{3\Delta_{i,j}}{2}, \cdots\}$.
\STATE Send the $\mathbf{m}_{i,j}$ by using $2C/\Delta_{i,j}$ bits to the PS.
\ENDFOR

\ENDFOR
\end{algorithmic}
\end{center} 

\begin{center}
\begin{algorithmic}
\textbf{Protocol in the PS:}
\FOR{$t\in[T]$}
\FOR{$i \in [N]$}
\FOR{$j \in [m]$}
\STATE Sample $v_{i,j} \sim \Gamma[3/2,1/2]$ and $U_j \sim \U\LB -\frac{\Delta_{i,j}}{2}, \frac{\Delta_{i,j}}{2} \RB$ using the shared randomness with client $i$.
\STATE Receive $\mathbf{m}_{i,j}$ and decode as $Q\LB(\fgAv{i}{t})_{j}+U_{i,j}\RB$
\STATE Estimate $(\fgAvEst{i}{t})_{j} = Q\LB(\fgAv{i}{t})_{j}+U_{i,j}\RB - U_j$
\ENDFOR
\ENDFOR
\STATE Average gradients $\g_t = \frac{1}{N} \sum_{i\in[N]} (\fgAv{i}{t})_{j}$
\STATE Update the model $\mathbf{w}_{t+1} = \mathbf{w}_{t} -\eta\g_t $
\STATE Broadcast $\mathbf{w}_{t+1}$ to all clients
\ENDFOR
\end{algorithmic}
\end{center}
\end{small}
\caption{Proposed Algorithm\label{alg:algorithm}}
\end{algorithm}

Using the common randomness shared between client $i$ and the PS, the same realizations of $V_{i,j}$'s and $U_{i,j}$'s generated by client $i$ can also be obtained by the PS. From $\mathbf{m}_{i,j}$, the PS can decode the value of $Q\LB(\fgAv{i}{t})_{j}+U_{i,j}\RB$, $\forall i \in [N]$ and $\forall j \in [m]$. Via subtractive dithering, the PS estimates $(\fgAv{i}{t})_{j}$ as $(\fgAvEst{i}{t})_{j}\triangleq Q\LB(\fgAv{i}{t})_{j}+U_{i,j}\RB - U_{i,j} = (\fgAv{i}{t})_{j} + U'_{i,j}$, where $U'_{i,j} \sim \U\LB-\frac{\Delta_{i,j}}{2}, \frac{\Delta_{i,j}}{2}\RB$. Finally, from $(\fgAvEst{i}{t})_{j}$, $\forall i \in [N]$, the PS calculates the global gradient average $\g$, and updates the global model accordingly. Then it broadcasts the new global model to the clients for the next round. The global gradient average satisfies central DP requirement as stated by the following theorem.

\begin{theorem}
    Global gradient average $\g$ is a noisy estimate of the averages of the local gradients such that 
    \begin{equation}
    \label{eq:average_cs}
        \g = \frac{1}{N} \sum_{i\in[N]} (\fgAv{i}{t})_i + \mathcal{N}\LB 0, \frac{\sigma^2}{N} \mathbf{I}_m\RB.
    \end{equation}
Hence, $\forall \varepsilon' > 0$, $\g$ satisfies $(\varepsilon,\delta)$-DP in sample-level against clients for $\varepsilon = \log\LB1+p(e^{\varepsilon'}-1)\RB
$
and 
\begin{equation}
\label{eq:dp_guarantee}
\delta=p\cdot\Phi\LB \frac{C}{\sigma B\sqrt{N}}-\frac{\varepsilon'\sigma B\sqrt{N}}{2C}\RB
-p\cdot e^{\varepsilon'}\Phi\LB-\frac{C}{\sigma B\sqrt{N}}-\frac{\varepsilon'\sigma B\sqrt{N}}{2C}\RB,    
\end{equation}
where $\Phi$ denotes the CDF of the standard normal distribution.
\end{theorem}

\begin{proof}
    For $j^{th}$ element of $(\fgAv{i}{t})_i$, $i \in [N]$, since $\Delta_{i,j} = 2\sigma\sqrt{v_{i,j}}$, $U'_{i,j}$ is distributed as $\U(-\sigma\sqrt{v_{i,j}}, \sigma\sqrt{v_{i,j}})$. Since $v_{i,j}$ is a sample from $\Gamma[3/2,1/2]$, according to \cref{lem:uniform_mixture}, $U'_{i,j}$ is distributed as $\mathcal{N}(0,\sigma^2)$ when the realization of $v_{i,j}$ is unknown. Since this is the case for any client $k\neq i$, the quantization noise of one client is normally distributed from the perspective of any other client. When we consider the averaging operation at the CS, we obtain \cref{eq:average_cs}. For the $(\varepsilon,\delta)$-DP guarantee, we use the following lemma, which is the joint statement of Theorem 8 in \cite{balle2018improving} and Theorem 8 in \cite{balle2018privacy}.

    \begin{lemma}
    \label{lem:gaussian_mechanism}
        Let $f$ be a function satisfying $\norm{2}{f(X)-f(X')}\leq L$ and $g$ be the Poisson sampling function with propability $p$. Then, for any $\varepsilon \geq 0$ and $\delta \in [0,1]$, the subsampled Gaussian mechanism $M(X) = f\circ g(X) + \mathcal{N}(0, \sigma'^2)$ is $(\varepsilon,\delta)$-DP if and only if, $\forall \varepsilon' >0$,  $\varepsilon = \log\LB1+p(e^{\varepsilon'}-1)\RB$ and

        \begin{equation}
        \label{eq:gaussian_mechanism}
            p\cdot \Phi\LB \frac{L}{2\sigma'}-\frac{\varepsilon'\sigma'}{L}\RB 
-p\cdot e^{\varepsilon'}\Phi\LB-\frac{L}{2\sigma'}-\frac{\varepsilon'\sigma'}{L}\RB \leq \delta.
        \end{equation}
    \end{lemma}

Since every client divides the sum of its sample gradients by $B$, and the PS also averages them over $N$ clients, the effect of the gradient of a single sample is at most $C/(BN)$. This makes the $L_2$ sensitivity of the signal $\frac{1}{N} \sum_{i\in[N]} (\fgAv{i}{t})_i$ to be $2C/(BN)$. If we put this quantity in place of $L$ in \cref{eq:gaussian_mechanism}, and set $\sigma'=\sigma/\sqrt{N}$ from \cref{eq:average_cs}, we obtain \cref{eq:dp_guarantee}.
\end{proof}

\begin{table}[th]
    \centering
    \caption{Experimental Results}
        \begin{tabular}{llcccccc}
    \toprule

    Dataset & Baseline & Final $\varepsilon$ & Accuracy & Epochs & $\sigma / \sqrt{N}$ & $C$ & $BN$\\ 
    \midrule
 \multirow{3}{*}{MNIST}  & Proposed Scheme & $1.45$ & $91.39{\scriptstyle \pm 1.04}$ & $10$ & $0.05$ & $2.0$ & \multirow{3}{*}{$32$}\\
 & Uncomp-DP & $1.45$ & $91.97 {\scriptstyle \pm 0.76}$& $10$ & $0.05$ & $2.0$ \\
 & Non-private & $\infty$ & $ 98.90{\scriptstyle \pm 0.09}$& $10$ & $0$ & $\infty$ \\
 \midrule
  \multirow{3}{*}{EMNIST}  & Proposed Scheme & $0.95$ & $70.17 {\scriptstyle \pm 0.26}$ & $10$ & $0.05$ & $2.0$& \multirow{3}{*}{$32$}\\
 & Uncomp-DP & $0.95$ & $70.02 {\scriptstyle \pm 0.47}$&  $10$ & $0.05$ & $2.0$\\
 & Non-private & $\infty$ & $85.11 {\scriptstyle \pm 0.13}$&  $10$ & $0$ & $\infty$\\
 \midrule
  \multirow{3}{*}{CIFAR-10}  & Proposed Scheme & $7.03$ & $51.66 {\scriptstyle \pm 0.03}$  & $100$ & $0.01$ & $1.0$ & \multirow{3}{*}{$64$}\\
 & Uncomp-DP & $7.03$ & $50.92 {\scriptstyle \pm 1.35}$& $100$ & $0.01$  &$1.0$\\
 & Non-private & $\infty$ & $80.95 {\scriptstyle \pm 0.64}$& $50$ & $0$ & $\infty$\\
\bottomrule
    \end{tabular}%
    \label{tab:results}
    \end{table}%

\section{Experiments}

In this section, we conduct numerical experiments using the proposed method and two other baselines, namely \textit{Uncomp-DP} and \textit{non-private}. In both baselines, we use an uncompressed transmission from the clients to the PS using double precision. In Uncomp-DP, the PS adds the required amount of noise to the sum of received client updates to achieve the target $(\varepsilon,\delta)$-DP guarantees. On the other hand, in the non-private case, we do not impose any privacy requirements and the PS only averages the client updates. 

We evaluate the proposed scheme and the baselines on the MNIST, EMNIST (\textit{ByClass} partition) and CIFAR-10 datasets. We employ LeNet architecture \cite{lecun1998gradient} for MNIST and EMNIST, and ResNet-18 \cite{he2016deep} for CIFAR-10. For DP training, we use the Opacus library \cite{yousefpour2021opacus} and employ the privacy accounting techniques based on Renyi-DP \cite{abadi2016deep, mironov2019r}. To simulate FL, we evenly distribute the dataset among $N$ clients, i.e., same $|\mathcal{D}_i|$'s $\forall i \in [N]$. We determine the Poisson sampling probability $p$ by simply dividing the expected total batch size at the PS, i.e., $BN$, by the number of data points in the whole dataset, $|\mathcal{D}_i|N$. We keep $BN$ constant across different $N$'s and we consider it as a hyperparameter to tune independent of $N$. Similarly, we tune the value of $\sigma$ so that $\sigma /\sqrt{N}$ remains constant so that the same DP guarantees hold regardless of the number of clients. While, for the non-private cases, we train until convergence, for the private cases, we determine the number of epochs to reach reasonable privacy and accuracy levels. 

We repeat each experiment 10 times and report the average accuracies in \cref{tab:results} along with the final $\varepsilon$ value after composition for $\delta=10^{-6}$, length of the gradient vector per round, communication cost per round and some related training hyperparameters. In general, in the experiments with all three datasets, we observe that the accuracies of the proposed method and the Uncomp-DP match. This verifies our theory claiming that subtractive dithering quantization is equivalent to adding Gaussian noise at the PS. Hence, the privacy accounting of these methods also matches. 

Since we tune $B$ and $\sigma$ so that $\sigma /\sqrt{N}$ and $BN$ remain constant with $N$, the accuracies of our experiments does not depend on the number of clients involved, which means that the same accuracies in \cref{tab:results} applies to different $N$ values. However, since now $\sigma$ depends on $N$, the communication cost per client increases with the number of clients involved. Fortunately, we observe that the communication cost scales logarithmically with the number of clients; and hence, even with a very large numbers of clients, our scheme still uses significantly less communication. We plot the average communication cost per gradient element via numerical simulations in \cref{fig:comm_vs_clients}. We observe that in MNIST and EMNIST experiments per-element costs match since we use the same $C$ and $\sigma$ parameters. For CIFAR-10, to have good accuracy, we tune the hyperparameters differently, and hence, we end up with a slightly larger communication cost. In all the cases, however, compared to double precision, which uses 64 bits per element, we use $12$ to $5.5$ times less communication depending on the number of clients without sacrificing accuracy. This observation shows that the proposed scheme saves a significant amount of communication for free in DP training.


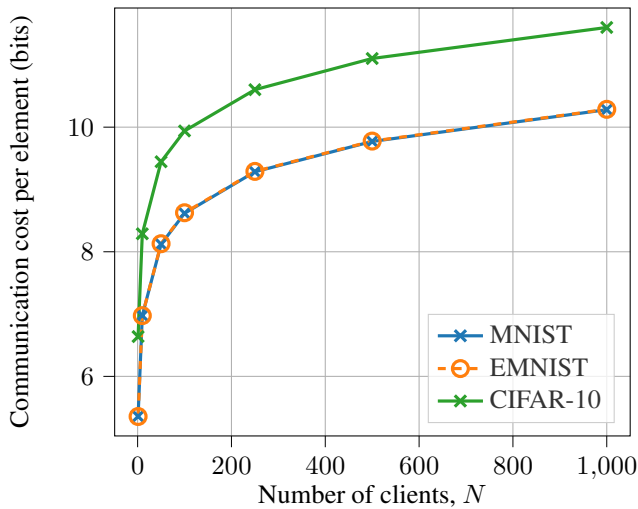
\begin{figure}[t]
\centering
\begin{tikzpicture}

\definecolor{color0}{rgb}{0.12156862745098,0.466666666666667,0.705882352941177}
\definecolor{color1}{rgb}{1,0.498039215686275,0.0549019607843137}
\definecolor{color2}{rgb}{0.172549019607843,0.627450980392157,0.172549019607843}

\begin{axis}[
legend cell align={left},
legend style={
  fill opacity=0.8,
  draw opacity=1,
  text opacity=1,
  at={(0.97,0.03)},
  anchor=south east,
  draw=white!80!black
},
tick align=outside,
tick pos=left,
x grid style={white!69.0196078431373!black},
xlabel={Number of clients, $N$},
xmin=-48.95, xmax=1049.95,
xtick style={color=black},
y grid style={white!69.0196078431373!black},
ylabel={Communication cost per element (bits)},
ymin=5.04389307498932, ymax=11.912415766716,
ytick style={color=black},
grid
]
\addplot [very thick, color0, mark=x, mark size=3, mark options={solid}]
table {%
1 5.35832643508911
10 6.97735548019409
50 8.12100982666016
100 8.6168909072876
250 9.28363990783691
500 9.77060699462891
1000 10.280421257019
};
\addlegendentry{MNIST}
\addplot [very thick, dashed, color1, mark=o, mark size=3, mark options={solid}]
table {%
1 5.35609865188599
10 6.97524881362915
50 8.12940692901611
100 8.62639713287354
250 9.2899112701416
500 9.77639961242676
1000 10.2861642837524
};
\addlegendentry{EMNIST}
\addplot [very thick, color2, mark=x, mark size=3, mark options={solid}]
table {%
1 6.6387243270874
10 8.28881645202637
50 9.44507122039795
100 9.93927097320557
250 10.6015644073486
500 11.1020011901855
1000 11.6002101898193
};
\addlegendentry{CIFAR-10}
\end{axis}

\end{tikzpicture}
\caption{Communication cost vs. number of clients for the proposed scheme. \label{fig:comm_vs_clients}}
\end{figure}

\section{Conclusion}
Through both theoretical analysis and experimental demonstrations, we have shown that using subtractive dithering quantization in the trusted aggregator model of FL can produce the same level of DP and accuracy guarantees as Gaussian noise addition, while utilizing fewer communication resources. This technique may prove useful in speeding up privacy-sensitive learning in communication-scarce scenarios such as edge training or time-critical industrial applications. Although the trusted aggregator model has many real-world applications, one possible area of future exploration is extending our methods to situations where trust in the PS is difficult to achieve. Additionally, exploring the possibility of extending the proposed technique to joint quantization would be an interesting future research direction.



 \bibliographystyle{IEEEbib}
 \bibliography{refs}

\end{document}